\documentclass[wcp]{jmlr}

\usepackage{longtable}% for long tables

\usepackage{booktabs}
\usepackage{nicefrac}
\usepackage{multirow}
\usepackage{bbm}
\usepackage{wrapfig}

\usepackage{algorithm}
\usepackage{algorithmic}

\makeatletter
\let\Ginclude@graphics\@org@Ginclude@graphics 
\makeatother

\jmlrvolume{157}
\jmlryear{2021}
\jmlrworkshop{ACML 2021}

\title[BRAC+: Improved Behavior Regularized Offline Reinforcement Learning]{BRAC+: Improved Behavior Regularized Actor Critic for Offline Reinforcement Learning}

  \author{\Name{Chi Zhang} \Email{zhan527@usc.edu}\\
  \Name{Sanmukh Rao Kuppannagari} \Email{kuppanna@usc.edu}\\
  \Name{Viktor K Prasanna} \Email{prasanna@usc.edu}\\
  \addr Hughes Aircraft Electrical Engineering Center, 3740 McClintock Ave, Los Angeles, CA 90089
 }

\editors{Vineeth N Balasubramanian and Ivor Tsang}

\DeclareMathOperator*{\st}{s.t.\ }
\DeclareMathOperator*{\kl}{\mathcal{D}_{\text{KL}}}
\DeclareMathOperator*{\argmax}{arg\,max}
\DeclareMathOperator*{\argmin}{arg\,min}

\graphicspath{{./figures/}}

\begin{document}

\maketitle

\begin{abstract}
Online interactions with the environment to collect data samples for training a Reinforcement Learning (RL) agent is not always feasible due to economic and safety concerns. The goal of Offline Reinforcement Learning is to address this problem by learning effective policies using previously collected datasets. Standard off-policy RL algorithms are prone to overestimations of the values of out-of-distribution (less explored) actions and are hence unsuitable for Offline RL. Behavior regularization, which constraints the learned policy within the support set of the dataset, has been proposed to tackle the limitations of standard off-policy algorithms. In this paper, we improve the behavior regularized offline reinforcement learning and propose BRAC+. First, we propose quantification of the out-of-distribution actions and conduct comparisons between using Kullback–Leibler divergence versus using Maximum Mean Discrepancy as the regularization protocol. We propose an analytical upper bound on the KL divergence as the behavior regularizer to reduce variance associated with sample based estimations.
Second, we mathematically show that the learned Q values can diverge even using behavior regularized policy update under mild assumptions. This leads to large overestimations of the Q values and performance deterioration of the learned policy.
To mitigate this issue, we add a gradient penalty term to the policy evaluation objective. By doing so, the Q values are guaranteed to converge. On challenging offline RL benchmarks, BRAC+  outperforms the baseline behavior regularized approaches by $40\%\sim 87\%$ and the state-of-the-art approach by $6\%$.
\end{abstract}

\section{Introduction}\label{sec:intro}

Reinforcement Learning (RL) has shown great success in a wide range of applications including board games \citep{alphago}, energy systems \citep{chi_buildsys19}, robotics \citep{rl_robots_nn}, recommendation systems \citep{recommendation_rl}, etc. The success of RL relies heavily on extensive online interactions with the environment for exploration. However, this is not always feasible in the real world as it can be expensive or dangerous~\citep{offlineRL_tutorial}. 

Offline RL, also known as batch RL, avoids online interactions with the environment by learning from a static dataset that is collected in an offline manner~\citep{offlineRL_tutorial}. While standard off-policy RL algorithms~\citep{dqn, ddpg, sac} can, in theory, be employed to learn from an offline data. In practice, they perform poorly due to distributional shift between the behavior policy (probability distribution of actions conditioned on states as observed in the dataset) of the collected dataset and the learned policy~\citep{offlineRL_tutorial}. The distributional shift manifests itself in form of overestimation of the out-of-distribution (OOD) actions, leading to erroneous Bellman backups.

Prior works tackle this problem via behavior regularization~\citep{bcq, bear,wu2019behavior,behavior_prior}. This ensures that the learned policy stays ``close" to the behavior policy. This is achieved by adding a regularization term that calculates the ``divergence" between the learned policy and the behavior policy. Kernel Maximum Mean Discrepancy (MMD) \citep{mmd}, Wasserstein distance and KL divergence are widely used~\citep{wu2019behavior}. The regularization term is either fixed~\citep{wu2019behavior}, or tuned via dual gradient descent~\citep{bear}, or applied using a trust region objective \citep{behavior_prior}.

In this paper, we propose improvements to the Behavior Regularized Actor Critic (BRAC) algorithm presented in \citep{wu2019behavior} and propose BRAC+. 
The key idea of behavior regularization is that the learned policy only takes actions that have high probability in the dataset. Under this criteria, we compare KL divergence and Maximum Mean Discrepancy (MMD) for behavior regularization. We show that MMD may erroneously assign low penalty to out-of-distribution actions when the behavior policy is multi-modal, leading to catastrophic failure. 
% Thus, BRAC+ uses forward KL divergence as behavior regularization function for multi-modal datasets. However, for ill-conditioned behavior policy density (e.g. Dirac delta function), a small change in the policy distribution leads to large increase of the divergence. Datasets with limited human demonstrations in a large action space \cite{d4rl} exhibit this behavior. In this case, MMD is superior to KL divergence due to its well-conditioned gradients. Thus, for such cases, BRAC+ uses MMD as behavior regularizer. We propose a technique that determines which behavior regularization to use depending upon the entropy of the dataset. 
Moreover, sample-based estimation of the KL divergence used by current techniques~\citep{wu2019behavior} is both computationally expensive and prone to high variance. Therefore, to reduce variance, we derive an analytical upper bound on the KL divergence measure as the regularization term in the objective function. 
Next, we mathematically show that the learned Q values can diverge even using behavior regularized policy update under mild assumptions. This is due to the erroneous generalization of the neural network, which leads to large overestimations of the Q values and performance deterioration of the learned policy. To mitigate this problem, we add a gradient penalty term to the policy evaluation objective. By doing so, the Q values are guaranteed to converge.

Our experiments suggest that BRAC+ outperforms the baseline behavior regularized approaches by $40\%\sim 87\%$ and the state-of-the-art approach by $6\%$ on challenging D4RL benchmarks \citep{d4rl}.

\section{Background}
\paragraph{Markov Decision Process}
RL algorithms aim to solve Markov Decision Process (MDP) with unknown dynamics. A Markov decision process \citep{rl_intro} is defined as a tuple $<\mathcal{S}, \mathcal{A}, R, P, \mu>$, where $\mathcal{S}$ is the set of states, $\mathcal{A}$ is the set of actions, $R(s, a, s'):\mathcal{S}\times\mathcal{A}\times\mathcal{S}\rightarrow \mathbb{R}$ defines the intermediate reward when the agent transitions from state $s$ to $s'$ by taking action $a$, $P(s'|s,a):\mathcal{S}\times\mathcal{A}\times\mathcal{S}\rightarrow [0, 1]$ defines the probability when the agent transitions from state $s$ to $s'$ by taking action $a$, $\mu: \mathcal{S}\rightarrow[0, 1]$ defines the starting state distribution. The objective of reinforcement learning is to select policy $\pi: \mathcal{\mu}\rightarrow P(A)$ to maximize the following objective: 
\begin{equation}
J(\pi)= \underset{\substack{s_0\sim\mu,a_t\sim \pi(\cdot|s_t)\\s_{t+1}\sim P(\cdot|s_t,a_t)}}{\mathbb{E}}[\sum_{t=0}^{\infty}\gamma^t R(s_t,a_t,s_{t+1})]
\label{eq:mdp_objective}
\end{equation}
The Q function under policy $\pi$ is defined as $Q^{\pi}(s_t,a_t)=\mathbb{E}_{\pi}\sum_{t'=t}^{\infty}r(s_t',a_t')$.

\paragraph{Off-policy Reinforcement Learning}
Modern deep off-policy reinforcement learning algorithms such as DQN \citep{dqn}, SAC \citep{sac} and TD3 \citep{td3} optimize Equation~\ref{eq:mdp_objective} by learning the Q function using data stored in a replay buffer $\mathcal{D}$:
\begin{align}
    \label{eq:policy_evaluation}
    \min_{\psi}\mathbb{E}_{(s,a,s',r)\sim\mathcal{D}} [(Q^{\pi}_{\psi}(s,a) - (r(s,a) + \gamma \mathbb{E}_{{a'\sim \pi_{\theta}}}Q^{\pi}_{\psi^{'}}(s',a')))]^2\qquad \text{(Policy Evaluation)}
\end{align}
where the Q function and the policy $\pi$ are approximated by neural networks parameterized by $\psi$ and $\theta$, respectively. $\psi^{'}$ denotes the parameters of the target Q network~\citep{dqn}. Then, the policy is trained to maximize the learned Q function:
\begin{align}
    \label{eq:policy_update}
    \max_{\theta}\mathbb{E}_{s\sim\mathcal{D}}[Q^{\pi}_{\psi}(s,\pi_{\theta})]\qquad\qquad \text{(Policy Update)}
\end{align}
Although data collected from any policies can be used to perform off-policy learning in Equation~\ref{eq:policy_evaluation} and~\ref{eq:policy_update}, it is crucial that the updated policy keeps interacting with the environment to collect on-policy data so that erroneous generalization of the Q function on unseen states can be corrected \citep{offlineRL_tutorial}.

\paragraph{Behavior Regularized Actor Critic}
Behavior Regularized Actor Critic (BRAC) \citep{wu2019behavior} solves the offline reinforcement learning problem by augmenting the policy update step in Equation~\ref{eq:policy_update} as:
\begin{align}
    \label{eq:policy_update_brac}
    \max_{\theta}\mathbb{E}_{s\sim\mathcal{D}}[Q^{\pi}_{\psi}(s,\pi_{\theta})]\qquad  \st \mathbb{E}_{s\sim\mathcal{D}}[D(\pi_{\theta}(\cdot|s)||\pi_b(\cdot|s))]<\epsilon \qquad \text{(BRAC Policy Update)}
\end{align}
where $D$ is a distance measurement between the learned policy $\pi_{\theta}$ and the behavior policy $\pi_b$. 
The behavior policy is defined as the solution that maximizes the probability of actions conditioned on states in the dataset:
\begin{align}
    \label{eq:behavior_policy}
    \pi_b(a|s)=\argmax_{\pi_b} \mathbb{E}_{(s,a)\sim\mathcal{D}}[p(a|s)]
\end{align}
Common distance measurements used in previous works \citep{wu2019behavior,bear} include Maximum Mean Discrepancy (MMD) and KL divergence. BRAC-p \citep{wu2019behavior} uses the standard policy evaluation in Equation~\ref{eq:policy_evaluation} while BRAC-v \citep{wu2019behavior} also augments the policy evaluation. As they show similar empirical results, our method BRAC+ proposes improvements over the simpler version BRAC-p.

% \paragraph{Variational Auto-encoder}
% A variational auto-encoder (VAE) \citep{vae} is a generative model that aims to learn the data distribution $p(X)$ given a set of observations $\{x_i\}_{i=1}^{N}$. While directly optimizing $p(X)$ is intractable, we can optimize its evidence lower-bound (ELBO):
% \begin{align}
%     \label{eq:vae_obj}
%     \log p(X)\geq \mathbb{E}_{z\sim q(z)}[\log p(X|z)] - \kl(q(z)||p(z))
% \end{align}
% where $q(Z)$ is the variational distribution and $p(Z)$ is the prior. In VAE, $q(Z)$ is chosen as $q(Z|X)$ so that it is an auto-encoder. Optimizing Equation~\ref{eq:vae_obj} using gradient descent requires back-propagate the gradient through a sample operator. Fortunately, if the latent variable is a multivariate Gaussian distribution, we can use re-parametrization trick. The tightness of the upper bound is the KL divergence between the approximated posterior distribution and the true posterior distribution.

% \paragraph{Difference in training the behavior policy} \citep{bear,wu2019behavior} learns the behavior policy as a $\beta$-VAE \citep{beta_vae} with MSE reconstruction loss. This is equivalent to maximizing the log probability of a Gaussian distribution with fixed variance as the decoder output. However, the variational lower bound does not hold in $\beta$-VAE that breaks our derivation of the analytically KL divergence upper bound. In this paper, we learn the behavior policy as a regular VAE. The decoder outputs a Gaussian distribution with input-conditioned mean and variance.

\section{Problem Statement}\label{sec:problem_statement}
Given a fixed dataset $\mathcal{D} = \{(s_i,a_i,s'_i,r_i)\}_{i=1}^{N}$, the objective of offline reinforcement learning is to learn policy $\pi_{\theta}$ such that Equation~\ref{eq:mdp_objective} is maximized. In principle, any off-policy algorithms can be directly applied. The major challenge is the absence of on-policy data to correct the erroneous generalization of the Q function on unseen states \citep{offlineRL_tutorial}.

\section{Related Work}
We briefly summarize prior works in deep offline RL and discuss their relationship with our approach. Please refer to \citep{offlineRL_tutorial} for traditional batch RL approaches. As discussed in Section~\ref{sec:intro}, the fundamental challenge in learning from a static data is to avoid 
out-of-distribution actions \citep{offlineRL_tutorial}. This requires solving two problems: 1) estimation of the behavior policy, and 2) quantification of the out-of-distribution actions. We follow BCQ \citep{bcq}, BEAR \citep{bear} and BRAC \citep{wu2019behavior} by learning the behavior policy using a conditional Variational Auto-encoder (VAE) \citep{vae}. To avoid out-of-distribution actions, BCQ generates actions in the target values by perturbing the behavior policy. However, this is over-pessimistic in most of the cases. BRAC \citep{wu2019behavior} constrains the policy using various sample-based divergence measurements including Maximum Mean Discrepancy (MMD), Wasserstein distance and KL divergence with penalized policy update (BRAC-p) or policy evaluation (BRAC-v). BEAR \citep{bear} is an instance of BRAC with penalized policy improvement using MMD \citep{mmd}. Sample-based estimation is computationally expensive and suffers from high variance. In contrast, our method uses an analytical upper-bound of the KL divergence to constrain the distance between the learned policy and the behavior policy. It is both computationally efficient and has low variance. \citep{behavior_prior} solves trust-region objective instead of using penalty. 
For KL regularized policy improvement with fixed temperature, the optimal policy has a closed form solution \citep{critic_rr}. However, tuning the fixed temperature is difficult as it doesn't imply the actual divergence. On the contrary, automatical temperature tuning via dual gradient descent can bound the actual divergence value within a certain threshold. \citep{Fox2019TowardPU} presents a closed-form expression for the regularization coefficient that completely eliminates the bias in entropy-regularized value updates. However, the \texttt{softmax} operator introduced by the approach makes it hard to use in continuous action space.
CQL \citep{cql} avoids estimating the behavior policy by learning a conservative Q function that lower-bounds its true value. Hyperparameter search is another challenging problem in offline RL. \citep{batch_rl_hyperparameter_gradient} uses a gradient-based optimization to search the hyperparameter using held-out data. MOPO \citep{mopo} follows MBPO \cite{mbpo} with additional reward penalty on unreliable model-generated transitions. MBOP \citep{mbop} learns the dynamics model, the behavior policy and a truncated value function to perform online planning. \citep{morel} learns a surrogate MDP using the dataset, such that taking out-of-distribution actions transit to the terminal state. The out-of-distribution actions are detected using the agreement of the predictions from ensembles of the learned dynamics models.

\section{Improving Behavior Regularized Offline Reinforcement Learning}
In this section, we first quantify the criteria of the out-of-distribution actions and compare various existing approaches to satisfy the criteria. We argue that Kullback–Leibler divergence (KLD) is superior to Maximum Mean Discrepancy (MMD) in meeting the criteria. However, accurately estimation of the KL divergence requires large amounts of samples to reduce the variance of the estimator. To avoid the expensive sampling, we derive an upper bound on the KL divergence between the learned policy and the behavior policy that can be computed analytically. Lastly, we mathematically show that the difference between the Q values of the learned policy and the behavior policy can be arbitrarily large even if the KL divergence is small. This leads to the learned Q function diverging. To prevent it, we apply the gradient penalty when learning the Q function.

\subsection{Quantification of Out-of-distribution Actions}\label{sec:quan_ood_actions}
It is critical to define a criteria that accurately distinguishes between policies that generate in-distribution and out-of-distribution actions. 
The key insight that can be leveraged to define this criteria is that the learned policy should only take actions that have high probability in the dataset $\mathcal{D}$.
We say that an action is in-distribution if $\pi_b(a|s)>\epsilon$. 
Ideally, the learned policy $\pi_{\theta}$ should have positive probability for in-distribution actions and zero probability elsewhere. 
However, this is impossible for continuous policies represented as neural networks. In practice, we say a policy distribution $\pi_{\theta}$ only contains in-distribution actions if:
\begin{align}
  \label{eq:criteria}
  \mathbb{E}_{a\sim\pi_{\theta}(\cdot|s)}[\pi_b(a|s)]>\epsilon
\end{align}
where $\epsilon$ is the pre-defined threshold. For the rest of the paper, we use the terms ``a policy distribution contains in/out-of-distribution actions" and ``in/out-of-distribution actions" interchangeably. Based on the criteria, we compare two behavior regularization protocols that are used in previous works \citep{bear,wu2019behavior}.

\paragraph{Kernel MMD}
Kernel Maximum Mean Discrepancy (MMD) \citep{mmd} was first introduced in \citep{bear} to penalize the policy from diverging from the behavior policy:
\begin{align}
  \label{eq:mmd}
  \text{MMD}_{k}^{2}(\pi(\cdot|s),\pi_b(\cdot|s))=\underset{x,x'\sim\pi(\cdot|s)}{\mathbb{E}}[K(x,x')]
  -2\underset{\substack{x\sim\pi(\cdot|s),\\y\sim\pi_b(\cdot|s)}}{\mathbb{E}}[K(x,y)]
  +\underset{y,y'\sim\pi_b(\cdot|s)}{\mathbb{E}}[K(y,y')]
\end{align}
where $K$ is a kernel function. 
Symmetric kernel functions such as Laplacian and Gaussian kernels are typically used \citep{mmd}. 
Kernel MMD can effectively prevent out-of-distribution actions for single-modal datasets (e.g. The dataset is collected by a Gaussian policy) as shown in \citep{bear}. 
However, it fails on multi-modal datasets (e.g. The dataset is collected by several distinct Gaussian policies.). 
This is because actions with low MMD distance may have low probability density. 
We show an example in Figure~\ref{fig:divergence}. In the middle figure, the behavior policy is a mixture of two Gaussian distributions. To minimize the MMD distance between the learned Gaussian policy and the behavior policy, the mean of the Gaussian policy is around zero, which has low density. This causes out-of-distribution actions are preferred to in-distribution actions and the failure of the behavior regularized offline reinforcement learning.
%It makes kernel MMD too restrictive as discussed above.

% \paragraph{Wasserstein distance} The $p$-Wasserstein distance between probability measures $\mu$ and $\nu$ on $\mathbb{R}^d$ is defined as:
% \begin{align}
%   \label{eq:wasserstein}
%   W_p(\mu, \nu)=\text{inf}_{X\sim\mu, Y\sim \nu}(\mathbb{E}||X-Y||^p)^{\nicefrac{1}{p}},\qquad p\geq 1
% \end{align}
% The Wasserstein distances $W_p$ are proper distances in that they are nonnegative, symmetric in $X$ and $Y$, and satisfy the triangle inequality. 
%Again, as per the discussion above, $W_p$ is too restrictive.

\paragraph{KL divergence}
For two probability distributions $P$ and $Q$ on some probability space $\chi$, the KL divergence from $Q$ to $P$ is given as
\begin{align}
    \kl(P,Q) = \int_{x \sim \chi} P(x)\log \frac{P(x)}{Q(x)}dx
\end{align}
By constraining the KL divergence $\kl(\pi_{\theta}(\cdot|s), \pi_b(\cdot|s))<\epsilon_{\text{KL}}$, criteria~\ref{eq:criteria} is automatically satisfied if we fix the entropy of the learned policy: $\mathcal{H(\pi_{\theta})}=\mathcal{H}_0$. Estimating the KL divergence requires i) learning an explicit likelihood model from the dataset; and ii) generating a large number of samples to reduce the variance. We can achieve step i) by applying state-of-the-art deep generative models such as Conditional Normalizing Flow \citep{normalizing_flow_intro} or Variational Auto-encoder \citep{vae}. However, evaluating the likelihood of a large number of samples is computationally expensive.

\begin{figure*}
  \centering
  \includegraphics[width=\linewidth]{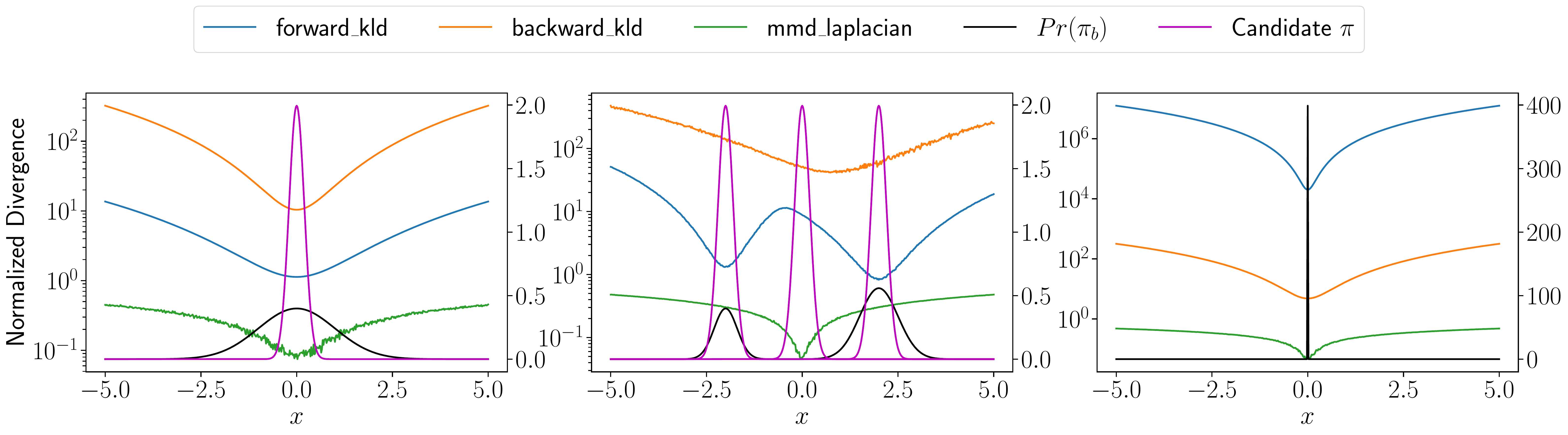}
  \vspace{-0.5cm}
  \caption{The black curve in the two figures represents the behavior distribution $\pi_b$. The blue, orange and green curves in the graph show forward KL, backward KL and MMD with Laplacian kernel between $\pi_b$ and $\pi_{\theta}=\mathcal{N}(x, \sigma)$, respectively, where $x$ is a variable between $[-10, 10]$ and $\sigma$ is fixed. In the left figure, $\pi_b=\mathcal{N}(0, 1)$. In the middle figure, $\pi_b$ is a mixture of two Gaussian distributions: $\mathcal{N}(-2.0, 0.3)$ and $\mathcal{N}(2.0, 0.5)$. The weight of each component is 0.3 and 0.7, respectively. $\sigma$ is set to 0.2. In the right figure, , $\pi_b=\mathcal{N}(0, 0.001)$. We show the candidate policies in purple for detailed comparison.}
  \label{fig:divergence}
\end{figure*}

\paragraph{Analytical KL divergence upper bound}
% All the existing behavior regularized methods estimate the divergence via samples \citep{bear, wu2019behavior}. While in theory it produces an unbiased estimator, it requires a large number of samples to reduce the variance.

% Our experimental results (Figure~\ref{fig:kld_mmd_result}), as well as \citep{cql} show that the performance deteriorates as the training proceeds longer. 
To reduce the variance and avoid computationally expensive sampling, we derive an upper bound on the KL divergence between the learned policy $\pi_{\theta}$ and the behavior policy $\pi_{b}$ that can be computed analytically.
Assume we learn $\pi_{b}$ using a Conditional VAE \citep{vae}) with latent variable $Z$. According to the evidence lower bound (ELBO), we obtain $\log \pi_{b}(a|s)\geq \mathbb{E}_{z\sim q(z|s,a)}[\log p(a|s,z)]-\kl(q(z|s,a)||p(z))$, where $q(z|s,a)$ is the approximated posterior distribution and $p(z)$ is the prior. Then, the KL divergence is bounded by:
\begin{align}
  \label{eq:kl_upperbound}
  \kl(\pi_{\theta}(\cdot|s)||&\pi_{b}(\cdot|s))=\mathbb{E}_{a\sim\pi_{\theta}}[\log \pi_{\theta}(a|s)-\log \pi_{b}(a|s)] \nonumber\\
  &\leq \mathbb{E}_{a\sim\pi_{\theta}}[\log \pi_{\theta}(a|s)]-\mathbb{E}_{a\sim\pi_{\theta}}[\mathbb{E}_{z\sim q(z|s,a)}[\log p(a|s,z)]-\kl(q(z|s,a)||p(z))] \nonumber\\
  &=\mathbb{E}_{a\sim\pi_{\theta}, z\sim q(z|s,a)}[\kl(\pi_{\theta}(\cdot|s)||p(\cdot|s,z))+\kl(q(z|s,a)||p(z))]\nonumber \\
  &=\mathcal{D}^{\text{upper}}_{\text{KL}}(\pi_{\theta}(\cdot|s)||\pi_{b}(\cdot|s))
  % &=\mathbb{E}_{z\sim p(z)}[\kl(\pi_{\theta}(\cdot|s)||p(\cdot|s,z))] \quad \text{at inference time } q(z)=p(z)
\end{align}
In our Conditional VAE implementation \citep{vae}, $p(\cdot|s,z)$ is the output distribution of the decoder. $q(z|s,a)$ is the output distribution of the encoder. Both of them are Gaussian distributions with parameterized mean and variance. $p(z)=\mathcal{N}(0,1)$ is the prior distribution. We parameterize the learned policy as a Gaussian distribution: $\pi_{\theta}(a|s)=\mathcal{N}(\mu_{\theta}(s), \sigma_{\theta}(s))$, where $\mu_{\theta}$ and $\sigma_{\theta}$ represents the output of a neural network. Since there is a closed-form solution of the KL divergence between two Gaussian distributions, Equation~\ref{eq:kl_upperbound} can be computed analytically.

\paragraph{Impact of the bias by using the analytical upper bound}
According to \citep{vae}, the bias term of the upper bound is $\kl(q(z|s,a)||p(z|s,a))$, where $p(z|s,a)$ denotes the true posterior distribution and $q(z|s,a)$ denotes the approximated posterior distribution. Using the linearity of the expectation, we obtain:
\begin{align}
  \resizebox{0.92\columnwidth}{!}{$
  \underset{s\sim\mathcal{D}}{\mathbb{E}}[{\kl(\pi_{\theta}(\cdot|s)||\pi_{b}(\cdot|s))}]=\underset{s\sim\mathcal{D}}{\mathbb{E}}[\mathcal{D}^{\text{upper}}_{\text{KL}}(\pi_{\theta}(\cdot|s)||\pi_{b}(\cdot|s))]-\underset{s\sim\mathcal{D}}{\mathbb{E}}[\kl(q(z|s,a)||p(z|s,a))]$
  }
\end{align}
For a fixed dataset and the trained conditional VAE, $\mathbb{E}_{s\sim\mathcal{D}}[\kl(q(z|s,a)||p(z|s,a))]$ is a fixed value. Thus, we can directly set the threshold of the upper bound without the necessity to estimate the bias because the threshold of the true KL divergence is also a hyperparameter.

Based on the aforementioned arguments, the final policy update step of our improved behavior regularized actor critic is:
\begin{align}
  \label{eq:policy_update_bracp}
  \max_{\theta}\mathbb{E}_{s\sim\mathcal{D}}[Q^{\pi}_{\psi}(s,\pi_{\theta})]\ \st \ \mathbb{E}_{s\sim\mathcal{D}}[{\mathcal{D}^{\text{upper}}_{\text{KL}}(\pi_{\theta}(\cdot|s)||\pi_{b}(\cdot|s))}]<\epsilon^{\text{upper}}_{\text{KL}}\ \text{and}\  \mathcal{H(\pi_{\theta})}=\mathcal{H}_0
\end{align}
We solve the constrained optimization problem in Equation~\ref{eq:policy_update_bracp} by using dual gradient descent \citep{sgd}.

\subsection{Gradient penalized policy evaluation}\label{sec:gp}
In this section, we argue that behavior regularization is not sufficient to learn stable policies from offline data due to arbitrarily large overestimation of the Q values.
We start by presenting Theorem~\ref{the:q_value_bound}:
\begin{theorem}
  \label{the:q_value_bound}
  Let the policy before and after the update in Equation~\ref{eq:policy_update_bracp} be $\pi_{old}$ and $\pi_{new}$. Assume the Q function generalizes such that $\pi_{old}$ is never a local maximizer of $\mathbb{E}_{a\sim\pi}[Q^{\pi}(s,a)]$, and after each policy evaluation in Equation~\ref{eq:policy_evaluation}, we have $\mathbb{E}_{a\sim\pi_b}[Q^{\pi_{\text{new}}}(s,a)-Q^{\pi_{\text{old}}}(s,a)]>\delta$, where $\delta=\sup_{a}|Q^{\pi_{\text{new}}}(s,a)-Q^{\pi_{\text{old}}}(s,a)|\cdot\sqrt{\nicefrac{\epsilon_{\text{KL}}}{2}}$ and $\epsilon_{\text{KL}}=\kl(\pi_{\text{new}}||\pi_b)$. Then, we have $\mathbb{E}_{a\sim\pi_{\text{new}}}[Q^{\pi_{\text{new}}}(s,a)]>\mathbb{E}_{a\sim\pi_{\text{old}}}[Q^{\pi_{\text{old}}}(s,a)]$.
\end{theorem}
\begin{proof}
  Please see Appendix~\ref{appendix:proof}.
\end{proof}
Theorem~\ref{the:q_value_bound} states that by iteratively updating the policy and the Q function using Equation~\ref{eq:policy_update_bracp} and Equation~\ref{eq:policy_evaluation}, under mild assumptions, the Q value of the current policy will monotonically increase after each iteration. 
Note that in practice, the assumption typically holds at states with low density. The gradients at those states are so large that even a small learning rate causes the increase of the Q values to be greater than $\delta$. That's why it's crucial to penalize the large gradients of the Q function.
Otherwise, it leads to infinitely large overestimation of the true Q value and thus a failure in learning a policy. This issue arises from the fact that we fit the Q function using $\mathbb{E}_{(s,a)\sim\mathcal{D}}[Q^{\pi}(s,a)]$. And the Q value of the policy $\mathbb{E}_{(s,a)\sim\pi}[Q^{\pi}(s,a)]$ relies on the generalization of the neural networks. Such generalization may be erroneous and may lead to divergence of the Q function. Note that in standard off-policy reinforcement learning, the erroneous generalization can be corrected due to the new \textbf{on-policy} data that is collected by environmental interactions~\citep{offlineRL_tutorial}.

%when we constantly collect \textbf{on-policy} data and add it to the dataset \citep{offlineRL_tutorial}.

\paragraph{Gradient penalized policy evaluation} To fix this issue, we need to control how the learned Q function generalizes. CQL \citep{cql} controls it by setting the maximum difference between the Q value of the learned policy and the behavior policy. In this work, we propose to penalize the gradient of the Q function, such that \textbf{the gradient of the Q function decreases as the distance between the learned policy and the behavior policy increases}. By doing so, there always exists a local maximizer of $\mathbb{E}_{a\sim\pi}[Q^{\pi}(s,a)]$. This guarantees the convergence of the Q function. Mathematically, we augment Equation~\ref{eq:policy_evaluation} as:
\begin{align}
  \label{eq:policy_improvement_gp}
  \resizebox{0.92\columnwidth}{!}{$
  \min_{\psi}\mathbb{E}_{(s,a,s',r)\sim\mathcal{D}} [(Q^{\pi}_{\psi}(s,a) - (r(s,a) + \gamma \mathbb{E}_{{a'\sim \pi_{\theta}}}Q^{\pi}_{\psi^{'}}(s',a')))]^2 +
  \lambda  \underset{a^{''}\sim\pi(\cdot|s)}{\mathbb{E}}(||\nabla_{a^{''}}Q^{\pi}_{\psi}(s,a^{''})||_2 f(\kl(\pi||\pi_b))$
  }
\end{align}
where $\pi$ is the learned policy and $f$ is a positive and monotonically increasing function. In practice, we find \texttt{softplus} function works very well. We tune the weight $\lambda$ using dual gradient descent.

\paragraph{Tradeoff analysis} Note that although gradient penalty prevents the Q function from divergence, it may be too conservative such that correct generalization is discarded. This may lead to worse performance compared with not using gradient penalty as shown in Section~\ref{sec:experiments} for the hopper-medium-replay task.

\begin{algorithm}[!t]
    \caption{BRAC+: Improved Behavior Regularized Actor Critic}
    \label{alg:tropo}
    \begin{algorithmic}[1]
      \STATE Train the behavior policy $\pi_{b}$ on the offline dataset $\mathcal{D}=\{(s_i,a_i,r_i,s^{'}_i)\}_{i=1}^{N}$ via maximum likelihood estimation: $\pi_{b}=\argmax_{\pi_{\beta}}\sum_{i=1}^{N}\log \pi_{\beta}(a_i|s_i)$
      \STATE Train initial policy: $\pi_{\theta}=\argmin_{\pi_{\theta}\in \Pi}\kl(\pi, \pi_{b})$
      \FOR{$e=1:E$}
        \FOR{$t=1:T$}
          \STATE Update Q network using Equation~\ref{eq:policy_improvement_gp}
          \STATE Update the policy using Equation~\ref{eq:policy_update_bracp}
          % \STATE Update $\alpha$ via dual gradient descent: $\alpha(s) \leftarrow \alpha(s) + \lambda_{\alpha}(D(\pi_{\theta}(\cdot|s), \pi_{\beta}(\cdot|s))-\epsilon)$
          % \STATE Update $\beta$ via dual gradient descent: $\beta\leftarrow \beta +\lambda_{\beta}(\mathcal{H}(\pi_{\theta}(\cdot|s))-\mathcal{H}_0)$
          \STATE Update the target network $\psi^{'}=\tau \psi + (1-\tau)\psi^{'}$
        \ENDFOR
      \ENDFOR
    \end{algorithmic}
\end{algorithm}

\section{Experiments}\label{sec:experiments}

Our experiments\footnote{Our implementation is available at \url{https://github.com/vermouth1992/bracp}.} aim to answer the following questions: 
\begin{enumerate}
    \item How does the performance of our methods compare with baseline behavior regularized approaches and the state-of-the-art model-free and model-based offline RL methods? (Section~\ref{sec:comparative_results})
    \item How does the use of analytical variational upper bound on KL divergence for regularization term compare with Maximum Mean Discrepancy (MMD)? (Section~\ref{subsec:ablation})
    \item Does the Q value diverge in real-world datasets if not using gradient penalty? Does the gradient penalized policy evaluation empirically guarantee the convergence of the Q function during the training? (Section~\ref{subsec:ablation})
\end{enumerate}
To answer these questions, we evaluate our methods on a subset of the D4RL \citep{d4rl} benchmark. We consider three locomotion tasks (hopper, walker2d, and halfcheetah) and four types of datasets: 1) random (rand): collect the interactions of a run of random policy for 1M steps to create the dataset, 2) medium (med): collect the interactions of a run of medium quality policy for 1M steps as the dataset, 3) medium-expert (med-exp): run a medium quality policy and an expert quality policy for 1M steps, respectively, and combine their interactions to create the dataset, 4) mixed (medium-replay): train a policy using SAC \citep{sac} until the performance of the learned policy exceeds a pre-determined threshold, and take the replay buffer as the dataset. 
% In addition, we consider more complex Adroit tasks \citep{adroit_env} that requires controlling a 24-DoF robotic hand, using limited data from human demonstrations.

We compare against state-of-the-art model-free and model-based baselines, including behavior cloning, BEAR \citep{bear} that constrains the learned policy within the support of the behavior policy using sampled MMD, offline SAC \citep{sac}, BRAC-p/v \citep{wu2019behavior} that constrains the learned policy within the support of the behavior policy using various sample-based $f$-divergences to penalize either the policy improvement (p) or the policy evaluation (v), CQL($\mathcal{H}$) \citep{cql} that learns a Q function that lower-bounds its true value. We also compare against model-based approaches including MOPO \citep{mopo} that follows MBPO \citep{mbpo} with additional reward penalties and MBOP \citep{mbop} that learns an offline model to perform online planning. 

\subsection{Comparative Results}\label{sec:comparative_results}

\begin{table*}[!ht]
    \centering
    \vspace{-2em}
    \caption{Results for OpenAI gym \citep{openai_gym} environments in the D4RL \citep{d4rl} datasets. For each task, we train for 1 million gradient steps and report the performance by running the policy obtained at the last epoch of the training for 100 episodes, averaged over 3 random seeds with standard deviation. Each number is the normalized score as proposed in \citep{d4rl}. Please refer to \citep{d4rl} for results on more baselines.}
    \vspace{-1em}
    \resizebox{\columnwidth}{!}{
    \begin{tabular}{c|cccccc|cc}
       \toprule
       \multirow{2}{*}{\textbf{Task Name}} & \multicolumn{6}{c|}{\bf Model-Free}  & \multicolumn{2}{c}{\bf Model-Based} \\
       \cmidrule{2-9}
        &\bf BC & \bf SAC-offline & \bf BEAR & \bf BRAC-p/v & \bf CQL($\mathcal{H}$) & \bf BRAC+ (Ours) & \bf MOPO & \bf MBOP \\
       \midrule
       halfcheetah-rand &2.1 &30.5 & 25.1 & 24.1/31.2 & \bf 35.4 & 29.7$\pm$1.4 &\bf 35.4$\pm$2.5 & 6.3$\pm$4.0 \\
       walker2d-rand & 1.6& 4.1 & 7.3 & -0.2/1.9 & 7.0 &  2.3$\pm$0.0  & \bf 13.6$\pm$2.6 & 8.1$\pm$ 5.5 \\
       hopper-rand  & 9.8& 11.3& 11.4 & 11.0/\bf 12.2 & 10.8 & \bf 12.2$\pm$0.1 & 11.7$\pm$0.4 & 10.8$\pm$ 0.3\\
       \midrule
       halfcheetah-med & 36.1&-4.3 & 41.7 & 43.8/46.3 & 44.4 & \bf 48.2 $\pm$ 0.3 &42.3$\pm$ 1.6 &  44.6$\pm$0.8 \\

       walker2d-med & 6.6&0.9 & 59.1 & 77.5/\bf 81.1 & 79.2 & 77.6$ \pm$ 0.8 &17.8$\pm $19.3 & 41.0 $\pm$ 29.4 \\

       hopper-med &29.0 & 0.8& 52.1 & 32.7/31.1 & \bf 58.0 & 39.1$\pm$3.5 & 28.0$\pm $12.4 & 48.8 $\pm$ 26.8\\

    %    \midrule
    %    halfcheetah-exp   & 108.2 & 3.8/-1.1 & 104.8 & $\pm$& N/A & N/A\\
    %    walker2d-exp   & 106.1 & -0.2/-0.0 & \bf 153.9 & \blue 100.9 $\pm$ 1.2 & N/A & N/A\\
    %    hopper-exp  & \bf 110.3 & 6.6/3.7 & \bf 109.9 & \blue \bf 112.2 $\pm$ 1.1 & N/A & N/A\\
       \midrule
       Average (single-modal) & 14.2 & 7.2 & 32.8 & 31.5/34.0 & \bf 39.1 & 34.9 & 24.8 & 26.6\\
       \midrule
       halfcheetah-med-exp & 35.8&1.8 & 53.4 & 44.2/41.9 & 62.4 & \bf 79.2$\pm$6.2 & 63.3$\pm$ 38.0 & \bf 105.9$\pm$ 17.8\\
       walker2d-med-exp & 6.4& -0.1 & 40.1 & 76.9/81.6 & \bf 98.7 & \bf 102.4$\pm$3.3 &44.6$\pm$ 12.9 & 70.2 $\pm$ 36.2 \\
       hopper-med-exp  &111.9 & 1.6& 96.3 & 1.9/0.8 & \bf 111.0 & \bf 111.2 $\pm$ 1.9 &23.7$\pm $6.0 & 55.1 $\pm$ 44.3\\
       \midrule
       halfcheetah-mixed  & 38.4&-2.4 & 38.6 & 45.4/\bf 47.7 & 46.2 & 46.2$\pm$1.3 & \bf 53.1$\pm$2.0 & 42.3 $\pm$ 0.9\\
       walker2d-mixed & 11.3& 1.9& 19.2 & -0.3/0.9 & 26.7 & \bf 46.1 $\pm$ 1.2 &\bf 39.0$\pm $9.6 & 9.7 $\pm$ 5.3 \\
       hopper-mixed & 11.8&  3.5& 33.7 & 0.6/0.6 & 48.6 & \bf 72.6$\pm$22.2  &\bf 67.5$\pm $24.7 & 12.4 $\pm$ 5.8 \\
       \midrule
       Average (multi-modal) & 35.9 & 1.1 & 46.9 & 28.1/28.9 & 65.6 & \bf 76.3 & 48.5 & 49.3\\
       \midrule
       Average (overall) & 25.1 & 4.1 & 39.8 & 29.8/31.4 & 52.4 &\bf 55.6 & 36.7 & 37.9\\
       \bottomrule
    \end{tabular}
    }
    \label{table:exp_results}
 \end{table*}

%  \begin{table*}[!ht]
%      \centering
%      \caption{Results for Adroit tasks with human demonstrations in the D4RL \citep{d4rl} datasets. The numbers are reported by following the same procedure as in Table~\ref{table:exp_results} except we run the policy obtained at the last epoch of training for 1000 episodes due to large variance across different runs.}
%      \vspace{-1em}
%      \resizebox{\columnwidth}{!}{
%      \begin{tabular}{c|ccccc|c}
%         \toprule
%         \bf Task Name & \bf BC & \bf SAC (offline) & \bf BEAR & \bf BRAC-p/v & \bf CQL($\mathcal{H}$) & \bf BRAC+ (Ours) \\
%         \midrule
%         pen-human & 34.4 &6.3 &  -1.0 & 8.1/0.6 & 37.5  & \bf 51.6$\pm$0.7  \\
%         hammer-human & 1.5 & 0.5 &0.3 & 0.3/0.2 & \bf 4.4  &  1.7$\pm$1.3 \\
%         door-human & 0.5 & 3.9&-0.3 & -0.3/-0.3 & \bf 9.9  &  1.5$\pm$0.5  \\
%         relocate-human & 0.0 & 0.0 &-0.3 & -0.3/-0.3 & \bf 0.20 &  0.05$\pm$0.03  \\
%         % pen-cloned & 56.9 & 26.5 & 39.2 & 40.3 & \\
%         % hammer-cloned & 0.8 & 0.3 & 2.1 & 5.7 \\
%         % door-cloned & -0.1 & -0.1 & 0.4 & 3.5 \\
%         % relocate-cloned & -0.1 & -0.3 & -0.1 & -0.1 \\
%         \bottomrule
%      \end{tabular}
%      }
%      \label{table:exp_results_adroit}
%  \end{table*}

 \paragraph{Performance on multi-modal datasets}
 We first compare the performance on multi-modal datasets i.e. med-exp and mixed datasets. Results shown in Table~\ref{table:exp_results} suggest that our method outperforms various model-free baselines on most of the multi-modal datasets, especially on hopper-mix and walker2d-mix by up to 1.5x. 
 Compared with BEAR \citep{bear}, our method improves the performance by $63\%$ due to the advantage of the KL divergence over the kernel MMD as discussed in Section~\ref{sec:quan_ood_actions}. 
 Compare with BRAC \citep{wu2019behavior}, our method improves the performance by $164\%$ due to the advantage of the analytical KL divergence upper bound.
 Compare with the state-of-the-art approach \citep{cql}, our method improves the performance by $16.3\%$.
 
 \paragraph{Performance on single-modal datasets}
 The performance of our method on single-modal (rand and med) dataset outperforms the baseline methods \citep{bear, wu2019behavior} as evident from Table~\ref{table:exp_results} by $3\%\sim 6\%$. Our method is worse than the state-of-the-art method \citep{cql} particularly on datasets collected by random policies.
 This is because randomly gathered dataset typically contains several discontinous modalities, where KL divergence based behavior regularization may stuck at a local optimum.

%  \paragraph{Performance on datasets with human demonstrations}
%  The performance on Adroit tasks is shown in Table~\ref{table:exp_results_adroit}. These tasks are substantially harder than OpenAI gym tasks due to limited training data in a high dimensional observation and action space. Our method makes non-trivial improvement over the behavior cloning. Compared with the state-of-the-art approaches, our approach is superior on the pen task. Figure~\ref{fig:adroit_performance} shows that the Q value is bounded when the gradient penalized policy evaluation technique is employed. On the contrary, the Q value without the gradient penalized policy evaluation increases exponentially. Note that we use MMD with Laplacian kernel for Adroit tasks. We observe that the KL-based regularization struggles with datasets collected with narrow behavior distributions (have large density within a tiny space and almost zero density anywhere else). In such a case, the KL divergence is sensitive to tiny policy changes, making gradient-based optimization hard to converge.

\subsection{Ablation Study}\label{subsec:ablation}
\begin{figure*}[!ht]
  \centering
  \includegraphics[width=\linewidth]{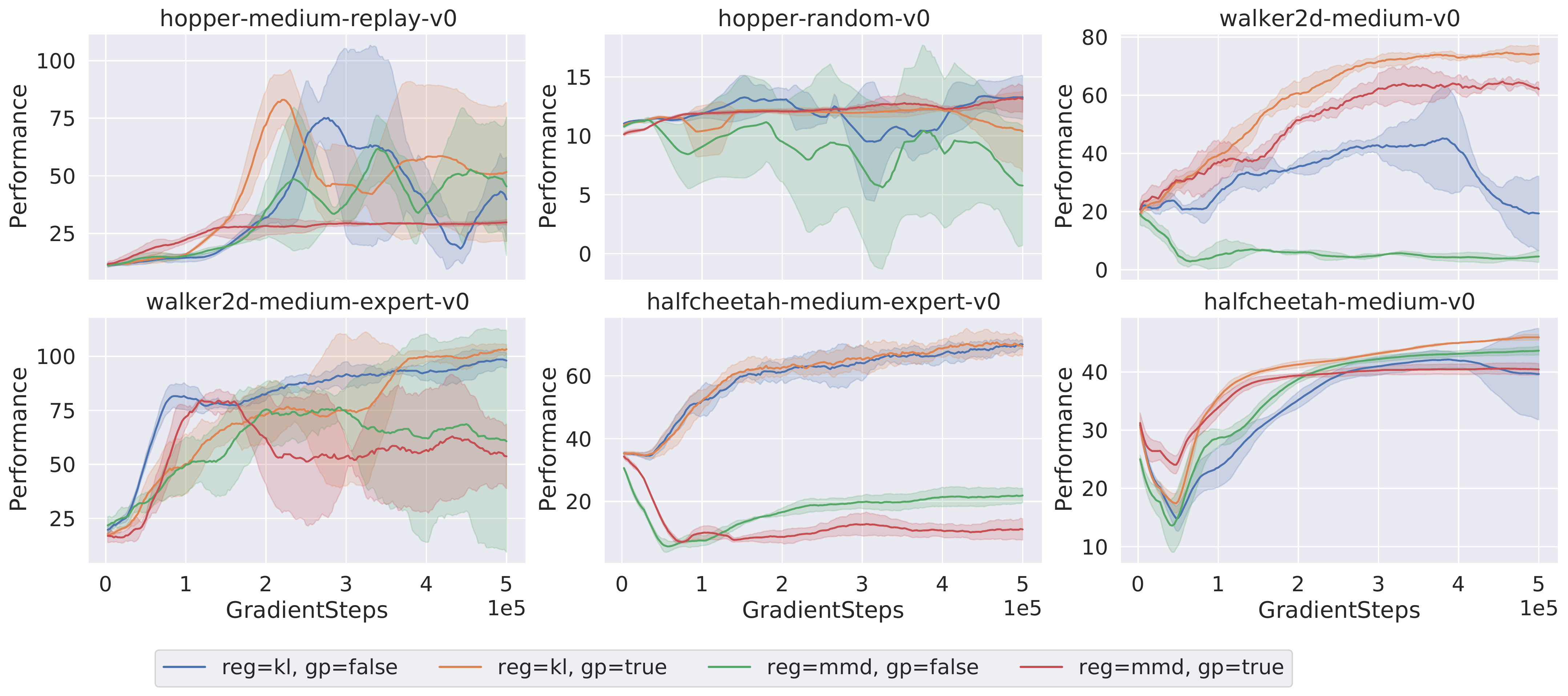}
  \vspace{-2em}
  \caption{Learning curve of various tasks using KL divergence and MMD as behavior regularization protocol, and with and without using gradient penalized policy evaluation. Each setting is repeated for 3 random seeds. The curve is the mean and the shaded area is the standard deviation. The curves are smoothed by a factor of 20. The hyperparameters can be found in Appendix~\ref{appendix:implementation}.}
  \label{fig:ablation_performance}
\end{figure*}
\begin{figure*}[!ht]
  \centering
  \includegraphics[width=\linewidth]{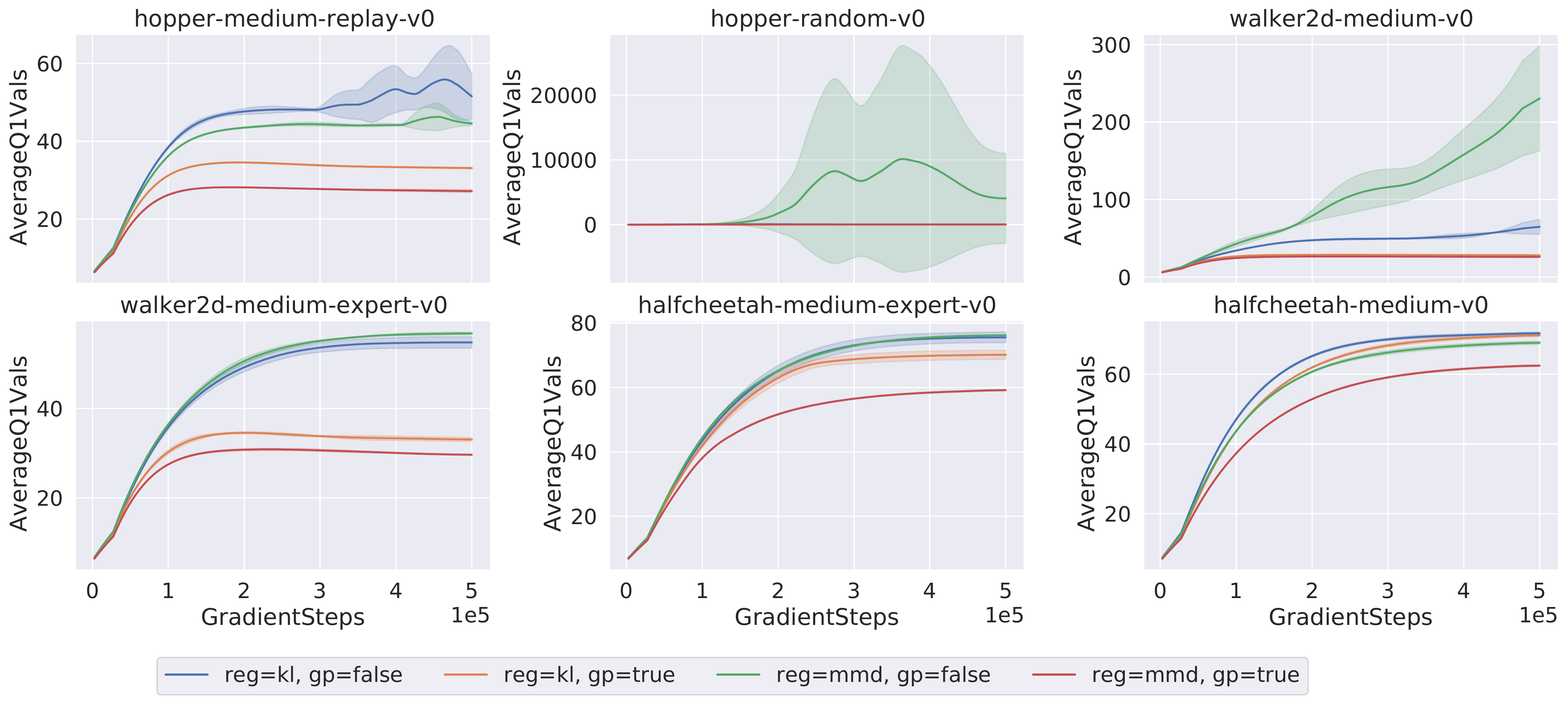}
  \vspace{-2em}
  \caption{The average Q values of all the states in the dataset during training. The setting is the same as in Figure~\ref{fig:ablation_performance}.}
  \label{fig:ablation_q}
\end{figure*}
To answer question (2) and (3), we conduct a thorough ablation study on BRAC+ on 6 tasks with different data collection policies (hopper-medium-replay (mixed), hopper-random, walker2d-medium, walker2d-medium-expert, halfcheetah-medium and halfcheetah-medium-expert).

\paragraph{Maximum Mean Discrepancy vs. KL divergence}
The results of using MMD versus KL divergence are shown in Figure~\ref{fig:ablation_performance}. In general, using KL divergence as the behavior regularization protocol provides better performance than using MMD. The performance is similar for hopper-medium-replay and hopper-random. The performance by using KL divergence is slightly better for walker2d-medium and halfcheetah-medium. The performance discrepancy is huge for halfcheetah-medium-expert and walker2d-medium-expert task and using MMD fails to learn better policies than the behavior cloning. This is because the medium-expert tasks contain multi-modal behavior policies as depicted in Figure~\ref{fig:divergence}, where the MMD assigns low penalty to out-of-distribution actions.

\paragraph{Gradient penalty vs. no gradient penalty}
From Figure~\ref{fig:ablation_performance}, we observe that the performance by using gradient penalty is better than not using gradient penalty except the hopper-medium-replay task with MMD as divergence. This is because gradient penalty prevents correct out-of-distribution generalization in this task. For other tasks, it is clear that without gradient penalty, the performance starts to decrease after some gradient steps due to the erroneous overestimation of the out-of-distribution actions. As shown in Figure~\ref{fig:ablation_q}, the gradient penalized policy evaluation provides more conservative Q values compared with no gradient penalty. It is also worth noticing that the Q values of hopper-medium-replay, hopper-random and walker2d-medium diverges when not using gradient penalty. This verifies our arguments in Section~\ref{sec:gp}.

\section{Discussions and Limitations}
We conjecture that the behavior-regularized approach with gradient penalty is not sufficient to tackle offline RL problems since it fully ignores the state distribution. To see this, we can create a dataset that only adds a few trajectories from an expert policy to a dataset collected by a low-quality policy. If the low-quality policy doesn't visit the ``good'' states in the expert policy (can't combine sub-optimal policies), behavior-regularized approach leads to a policy that imitates the expert policy. Such imitation is likely to fail due to compounding errors \citep{dagger}. The right approach for this dataset is to completely ignore expert trajectories and combine sub-optimal policies in the low-quality regions. To achieve this, we need to consider the state distribution as the density of the ``good'' states is very low.

\section{Conclusion}
In this paper, we improved the behavior regularized offline reinforcement learning by proposing a low-variance upper bound of the KL divergence estimator to reduce variance and gradient penalized policy evaluation such that the learned Q functions are guaranteed to converge. Our experimental results on challenging benchmarks illustrate the benefits of our improvements.

\acks{This work is supported by U.S. National Science Foundation (NSF) under award number 2009057 and U.S. Army Research Office (ARO) under award number W911NF1910362. We would like to thank Chen-Yu Wei for his help in the mathematical derivation.}

\bibliography{acml21}

\begin{thebibliography}{32}
\providecommand{\natexlab}[1]{#1}
\providecommand{\url}[1]{\texttt{#1}}
\expandafter\ifx\csname urlstyle\endcsname\relax
  \providecommand{\doi}[1]{doi: #1}\else
  \providecommand{\doi}{doi: \begingroup \urlstyle{rm}\Url}\fi

\bibitem[Agarap(2018)]{relu}
Abien~Fred Agarap.
\newblock Deep learning using rectified linear units (relu).
\newblock \emph{ArXiv}, abs/1803.08375, 2018.

\bibitem[Argenson and Dulac-Arnold(2020)]{mbop}
Arthur Argenson and Gabriel Dulac-Arnold.
\newblock Model-based offline planning.
\newblock \emph{ArXiv}, abs/2008.05556, 2020.

\bibitem[Brockman et~al.(2016)Brockman, Cheung, Pettersson, Schneider,
  Schulman, Tang, and Zaremba]{openai_gym}
Greg Brockman, Vicki Cheung, Ludwig Pettersson, Jonas Schneider, John Schulman,
  Jie Tang, and Wojciech Zaremba.
\newblock Openai gym, 2016.

\bibitem[Choi et~al.(2018)Choi, Ha, Hwang, Kim, Ha, and
  Yoon]{recommendation_rl}
Sungwoon Choi, Heonseok Ha, Uiwon Hwang, Chanju Kim, Jung-Woo Ha, and S.~Yoon.
\newblock Reinforcement learning based recommender system using biclustering
  technique.
\newblock \emph{ArXiv}, abs/1801.05532, 2018.

\bibitem[Fox(2019)]{Fox2019TowardPU}
R.~Fox.
\newblock Toward provably unbiased temporal-difference value estimation.
\newblock 2019.

\bibitem[Fu et~al.(2020)Fu, Kumar, Nachum, Tucker, and Levine]{d4rl}
Justin Fu, Aviral Kumar, Ofir Nachum, G.~Tucker, and Sergey Levine.
\newblock D4rl: Datasets for deep data-driven reinforcement learning.
\newblock \emph{ArXiv}, abs/2004.07219, 2020.

\bibitem[Fujimoto et~al.(2018{\natexlab{a}})Fujimoto, Hoof, and Meger]{td3}
Scott Fujimoto, H.~V. Hoof, and David Meger.
\newblock Addressing function approximation error in actor-critic methods.
\newblock \emph{ArXiv}, abs/1802.09477, 2018{\natexlab{a}}.

\bibitem[Fujimoto et~al.(2018{\natexlab{b}})Fujimoto, Meger, and Precup]{bcq}
Scott Fujimoto, David Meger, and Doina Precup.
\newblock Off-policy deep reinforcement learning without exploration.
\newblock \emph{CoRR}, abs/1812.02900, 2018{\natexlab{b}}.
\newblock URL \url{http://arxiv.org/abs/1812.02900}.

\bibitem[Gretton et~al.(2007)Gretton, Borgwardt, Rasch, Scholkopf, and
  Smola]{mmd}
Arthur Gretton, Karsten~M. Borgwardt, Malte~J. Rasch, Bernhard Scholkopf, and
  Alexander~J. Smola.
\newblock A kernel approach to comparing distributions.
\newblock In \emph{AAAI}, pages 1637--1641, 2007.
\newblock URL \url{http://www.aaai.org/Library/AAAI/2007/aaai07-262.php}.

\bibitem[Haarnoja et~al.(2018)Haarnoja, Zhou, Abbeel, and Levine]{sac}
Tuomas Haarnoja, Aurick Zhou, Pieter Abbeel, and Sergey Levine.
\newblock Soft actor-critic: Off-policy maximum entropy deep reinforcement
  learning with a stochastic actor.
\newblock \emph{CoRR}, abs/1801.01290, 2018.
\newblock URL \url{http://arxiv.org/abs/1801.01290}.

\bibitem[Janner et~al.(2019)Janner, Fu, Zhang, and Levine]{mbpo}
Michael Janner, Justin Fu, Marvin Zhang, and Sergey Levine.
\newblock When to trust your model: Model-based policy optimization.
\newblock \emph{CoRR}, abs/1906.08253, 2019.
\newblock URL \url{http://arxiv.org/abs/1906.08253}.

\bibitem[Kidambi et~al.(2020)Kidambi, Rajeswaran, Netrapalli, and
  Joachims]{morel}
Rahul Kidambi, Aravind Rajeswaran, Praneeth Netrapalli, and Thorsten Joachims.
\newblock Morel : Model-based offline reinforcement learning, 2020.

\bibitem[Kingma and Ba(2015)]{adam}
Diederik~P. Kingma and Jimmy Ba.
\newblock Adam: A method for stochastic optimization.
\newblock \emph{CoRR}, abs/1412.6980, 2015.

\bibitem[Kingma and Welling(2014)]{vae}
Diederik~P. Kingma and Max Welling.
\newblock Auto-encoding variational bayes.
\newblock \emph{CoRR}, abs/1312.6114, 2014.

\bibitem[Kobyzev et~al.(2020)Kobyzev, Prince, and
  Brubaker]{normalizing_flow_intro}
I.~Kobyzev, S.~Prince, and M.~Brubaker.
\newblock Normalizing flows: An introduction and review of current methods.
\newblock \emph{IEEE transactions on pattern analysis and machine
  intelligence}, 2020.

\bibitem[Kumar et~al.(2019)Kumar, Fu, Tucker, and Levine]{bear}
Aviral Kumar, Justin Fu, George Tucker, and Sergey Levine.
\newblock Stabilizing off-policy q-learning via bootstrapping error reduction.
\newblock \emph{CoRR}, abs/1906.00949, 2019.
\newblock URL \url{http://arxiv.org/abs/1906.00949}.

\bibitem[Kumar et~al.(2020)Kumar, Zhou, Tucker, and Levine]{cql}
Aviral Kumar, Aurick Zhou, G.~Tucker, and Sergey Levine.
\newblock Conservative q-learning for offline reinforcement learning.
\newblock \emph{ArXiv}, abs/2006.04779, 2020.

\bibitem[Lee et~al.(2020)Lee, Lee, Vrancx, Kim, and
  Kim]{batch_rl_hyperparameter_gradient}
Byung-Jun Lee, Jongmin Lee, Peter Vrancx, DongHo Kim, and Kee-Eung Kim.
\newblock Batch reinforcement learning with hyperparameter gradients.
\newblock 2020.

\bibitem[Levine et~al.(2020)Levine, Kumar, Tucker, and Fu]{offlineRL_tutorial}
Sergey Levine, Aviral Kumar, George Tucker, and Justin Fu.
\newblock Offline reinforcement learning: Tutorial, review, and perspectives on
  open problems.
\newblock \emph{ArXiv}, abs/2005.01643, 2020.

\bibitem[Lillicrap et~al.(2016)Lillicrap, Hunt, Pritzel, Heess, Erez, Tassa,
  Silver, and Wierstra]{ddpg}
Timothy~P. Lillicrap, Jonathan~J. Hunt, Alexander Pritzel, Nicolas Manfred~Otto
  Heess, Tom Erez, Yuval Tassa, David Silver, and Daan Wierstra.
\newblock Continuous control with deep reinforcement learning.
\newblock \emph{CoRR}, abs/1509.02971, 2016.

\bibitem[Lin(1992)]{rl_robots_nn}
Long-Ji Lin.
\newblock \emph{Reinforcement Learning for Robots Using Neural Networks}.
\newblock PhD thesis, USA, 1992.

\bibitem[Mnih et~al.(2013)Mnih, Kavukcuoglu, Silver, Graves, Antonoglou,
  Wierstra, and Riedmiller]{dqn}
Volodymyr Mnih, Koray Kavukcuoglu, David Silver, Alex Graves, Ioannis
  Antonoglou, Daan Wierstra, and Martin~A. Riedmiller.
\newblock Playing atari with deep reinforcement learning.
\newblock \emph{CoRR}, abs/1312.5602, 2013.
\newblock URL \url{http://arxiv.org/abs/1312.5602}.

\bibitem[Reid and Williamson(2009)]{pinskter}
Mark~D. Reid and Robert~C. Williamson.
\newblock Generalised pinsker inequalities.
\newblock \emph{CoRR}, abs/0906.1244, 2009.
\newblock URL \url{http://arxiv.org/abs/0906.1244}.

\bibitem[Robbins and Monro(1951)]{sgd}
H.~Robbins and S.~Monro.
\newblock A stochastic approximation method.
\newblock \emph{Annals of Mathematical Statistics}, 22:\penalty0 400--407,
  1951.

\bibitem[Ross et~al.(2010)Ross, Gordon, and Bagnell]{dagger}
St{\'{e}}phane Ross, Geoffrey~J. Gordon, and J.~Andrew Bagnell.
\newblock No-regret reductions for imitation learning and structured
  prediction.
\newblock \emph{CoRR}, abs/1011.0686, 2010.
\newblock URL \url{http://arxiv.org/abs/1011.0686}.

\bibitem[Siegel et~al.(2020)Siegel, Springenberg, Berkenkamp, Abdolmaleki,
  Neunert, Lampe, Hafner, and Riedmiller]{behavior_prior}
Noah Siegel, Jost~Tobias Springenberg, Felix Berkenkamp, Abbas Abdolmaleki,
  Michael Neunert, T.~Lampe, Roland Hafner, and Martin~A. Riedmiller.
\newblock Keep doing what worked: Behavioral modelling priors for offline
  reinforcement learning.
\newblock \emph{ArXiv}, abs/2002.08396, 2020.

\bibitem[Silver et~al.(2016)Silver, Huang, Maddison, Guez, Sifre, van~den
  Driessche, Schrittwieser, Antonoglou, Panneershelvam, Lanctot, Dieleman,
  Grewe, Nham, Kalchbrenner, Sutskever, Lillicrap, Leach, Kavukcuoglu, Graepel,
  and Hassabis]{alphago}
David Silver, Aja Huang, Christopher~J. Maddison, Arthur Guez, Laurent Sifre,
  George van~den Driessche, Julian Schrittwieser, Ioannis Antonoglou, Veda
  Panneershelvam, Marc Lanctot, Sander Dieleman, Dominik Grewe, John Nham, Nal
  Kalchbrenner, Ilya Sutskever, Timothy Lillicrap, Madeleine Leach, Koray
  Kavukcuoglu, Thore Graepel, and Demis Hassabis.
\newblock Mastering the game of go with deep neural networks and tree search.
\newblock \emph{Nature}, 529:\penalty0 484--503, 2016.
\newblock URL
  \url{http://www.nature.com/nature/journal/v529/n7587/full/nature16961.html}.

\bibitem[Sutton and Barto(2018)]{rl_intro}
Richard~S. Sutton and Andrew~G. Barto.
\newblock \emph{Reinforcement Learning: An Introduction}.
\newblock A Bradford Book, Cambridge, MA, USA, 2018.
\newblock ISBN 0262039249.

\bibitem[Wang et~al.(2020)Wang, Novikov, Zolna, Springenberg, Reed, Shahriari,
  Siegel, Merel, Gulcehre, Heess, and Freitas]{critic_rr}
Ziyu Wang, A.~Novikov, Konrad Zolna, Jost~Tobias Springenberg, Scott Reed,
  B.~Shahriari, N.~Siegel, Josh Merel, Caglar Gulcehre, Nicolas Heess, and
  N.~D. Freitas.
\newblock Critic regularized regression.
\newblock \emph{ArXiv}, abs/2006.15134, 2020.

\bibitem[Wu et~al.(2019)Wu, Tucker, and Nachum]{wu2019behavior}
Yifan Wu, George Tucker, and Ofir Nachum.
\newblock Behavior regularized offline reinforcement learning, 2019.

\bibitem[Yu et~al.(2020)Yu, Thomas, Yu, Ermon, Zou, Levine, Finn, and Ma]{mopo}
Tianhe Yu, Garrett Thomas, Lantao Yu, Stefano Ermon, James Zou, Sergey Levine,
  Chelsea Finn, and Tengyu Ma.
\newblock Mopo: Model-based offline policy optimization, 2020.

\bibitem[Zhang et~al.(2019)Zhang, Kuppannagari, Kannan, and
  Prasanna]{chi_buildsys19}
Chi Zhang, Sanmukh~R. Kuppannagari, Rajgopal Kannan, and Viktor~K. Prasanna.
\newblock Building hvac scheduling using reinforcement learning via neural
  network based model approximation.
\newblock In \emph{Proceedings of the 6th ACM International Conference on
  Systems for Energy-Efficient Buildings, Cities, and Transportation}, BuildSys
  '19, pages 287--296, New York, NY, USA, 2019. Association for Computing
  Machinery.
\newblock ISBN 9781450370059.
\newblock \doi{10.1145/3360322.3360861}.
\newblock URL \url{https://doi.org/10.1145/3360322.3360861}.

\end{thebibliography}

% \newpage

\appendix

\section{Proofs}\label{appendix:proof}
\subsection{Proof of Theorem~\ref{the:q_value_bound}}
\begin{proof}
    Let the current policy be $\pi_{\text{old}}$ and the Q function be $Q^{\pi_{\text{old}}}$. Since $\pi_{\text{old}}$ is never a local maximizer of $\mathbb{E}_{a\sim\pi}[Q^{\pi}(s,a)]$, update each policy update step defined in Equation~\ref{eq:policy_update_brac}, we obtain:
    \begin{align}
        \label{eq:policy_update_inequal}
        \mathbb{E}_{a\sim \pi_{\text{new}}}[Q^{\pi_{\text{old}}}(s,a)]>\mathbb{E}_{a\sim \pi_{\text{old}}}[Q^{\pi_{\text{old}}}(s,a)]
    \end{align}
    Let $\Delta Q(s,a)=Q^{\pi_{\text{new}}}(s,a)-Q^{\pi_{\text{old}}}(s,a)$. According to Pinsker’s inequality~\citep{pinskter}, we obtain:
    \begin{align}
        |\mathbb{E}_{a\sim\pi_{\text{new}}}[\Delta Q(s,a)]-\mathbb{E}_{a\sim\pi_b}[\Delta Q(s,a)]|\leq \sup_{a}|\Delta Q(s,a)|\sqrt{\frac{\kl(\pi_{\text{new}}||\pi_b)}{2}}
    \end{align}
    According to the assumption, we have $\mathbb{E}_{a\sim\pi_b}[\Delta Q(s,a)]>\sup_{a}|\Delta Q(s,a)|\sqrt{\frac{\kl(\pi_{\text{new}}||\pi_b)}{2}}$. Thus, we obtain:
    \begin{align}
        \label{eq:policy_evaluation_inequal}
        \mathbb{E}_{a\sim\pi_{\text{new}}}[\Delta Q(s,a)]> 0
    \end{align}
    Combining Equation~\ref{eq:policy_update_inequal} and Equation~\ref{eq:policy_evaluation_inequal}, we obtain:
    \begin{align}
        &\mathbb{E}_{a\sim\pi_{\text{new}}}[Q^{\pi_{\text{new}}}(s,a)] - \mathbb{E}_{a\sim\pi_{\text{old}}}[Q^{\pi_{\text{old}}}(s,a)]\nonumber\\
        &=\mathbb{E}_{a\sim\pi_{\text{new}}}[Q^{\pi_{\text{new}}}(s,a)] - \mathbb{E}_{a\sim\pi_{\text{old}}}[Q^{\pi_{\text{old}}}(s,a)] + \mathbb{E}_{a\sim\pi_{\text{new}}}[Q^{\pi_{\text{old}}}(s,a)] - \mathbb{E}_{a\sim\pi_{\text{new}}}[Q^{\pi_{\text{old}}}(s,a)] \nonumber\\
        &=(\mathbb{E}_{a\sim\pi_{\text{new}}}[Q^{\pi_{\text{old}}}(s,a)]-\mathbb{E}_{a\sim\pi_{\text{old}}}[Q^{\pi_{\text{old}}}(s,a)]) + (\mathbb{E}_{a\sim\pi_{\text{new}}}[Q^{\pi_{\text{new}}}(s,a)]-\mathbb{E}_{a\sim\pi_{\text{new}}}[Q^{\pi_{\text{old}}}(s,a)]) \nonumber \\
        &> \mathbb{E}_{a\sim \pi_{\text{new}}}[Q^{\pi_{\text{new}}}(s,a)-Q^{\pi_{\text{old}}}(s,a)]=\mathbb{E}_{a\sim\pi_{\text{new}}}[\Delta Q(s,a)]> 0
    \end{align}
    Thus, $\mathbb{E}_{a\sim\pi_{\text{new}}}[Q^{\pi_{\text{new}}}(s,a)] > \mathbb{E}_{a\sim\pi_{\text{old}}}[Q^{\pi_{\text{old}}}(s,a)]$.
\end{proof}

\section{Implementation Details}\label{appendix:implementation}
% \paragraph{Computation of the analytical KL upper bound}
% To reduce the variance when computing the analytical KL upper bound, we sample $L$ latent variable $z$ and compute the average of the $L$ KL upper bounds. In our experiments, we set $L=5$. Note that it doesn't reduce the bias of the upper bound.

\paragraph{Reward scaling}
Any affine transformation of the reward function does not change the optimal policy of the MDP. In our experiments, we rescale the reward to $[0, 1]$ as:
\begin{align}
    r'=(r-r_{min}) / (r_{max} - r_{min})
\end{align}
where $r_{max}$ and $r_{min}$ is the maximum and the minimum reward in the dataset.

% \paragraph{Entropy regularization}
% The KL divergence is the sum of negative entropy of the learned policy plus the cross entropy between the learned policy and the behavior policy: $\kl(\pi_{\theta}(\cdot|s), \pi_{\beta}(\cdot|s))=-\mathcal{H}(\pi_{\theta}(\cdot|s))+\mathcal{H}(\pi_{\theta}(\cdot|s), \pi_{\beta}(\cdot|s))$. When the learned policy distribution violates the KL constraints, the KL divergence between the policy distribution and the behavior distribution is decreased by the optimizer. This is equivalent to increasing the entropy of the learned policy and decreasing the cross entropy between the learned policy and the behavior policy. In soft actor-critic \citep{sac}, the minimum entropy of the learned policy is enforced to encourage exploration. However, due to the absence of exploration, stochastic policy with large entropy will sample out-of-distribution actions when computing the target Q values in Equation~\ref{eq:policy_evaluation}. If such values are overestimated, the policy will exploit the erroneous Q values when performing the policy improvement in Equation~\ref{eq:policy_improvement} and lead to failure, which can't be corrected without more data. Thus, we maintain the maximum entropy of the learned policy using the technique proposed in \citep{sac_algo_apps}.

\paragraph{Initialization}
If the dataset is collected using a narrow policy distribution in a high dimensional space (e.g. human demonstration), the constrained optimization problem using dual gradient descent finds it difficult to converge if random initialization is used for the policy network. To mitigate this issue, we start with a policy that has the minimum KL divergence with the behavior policy: $\pi_{\theta}=\argmin_{\pi_{\theta}\in \Pi}\kl(\pi, \pi_{\beta})$, where $\Pi$ represents a family of policy types. In this work, we consider $\Pi$ as Gaussian policies. Correspondingly, we initialize the Q network to $Q^{\pi_{\theta}}$.

% \paragraph{Policy Evaluation Masking}
% In our experiments, we observe that no matter how high the Lagrange multipliers of behavior regularization terms are, there always exists a small number of states where the KL divergence between the policy distribution and the behavior distribution is much larger than the threshold. We hypothesize that this issue arises because: 1) the capacity of the neural network is limited, 2) there are always some states in the dataset which are not sufficiently explored (e.g. the behavior distribution density is a Dirac function for them.). This leads to out-of-distribution actions if the training is continued for sufficiently long time. To mitigate this problem, we apply zero masking to prevent the target Q values evaluated at the out-of-distribution actions from updating the Q network:
% \begin{align}
%   Q_{\psi}=\argmin_{\psi} [(\mathbb{1}[\kl(\pi(\cdot|s'), \pi_{\beta}(\cdot|s'))]<{\epsilon'}_{\text{KL}})(Q_{\psi}(s,a) - (r(s,a) + \gamma \mathbb{E}_{{ a'\sim \pi_{\theta}}}Q_{\psi^{'}}(s',a')))]^2
% \end{align}
% where ${\epsilon'}_{\text{KL}}$ decides the hard constraints. In practice, we set ${\epsilon'}_{\text{KL}}=\epsilon_{\text{KL}}+\sigma_{\text{KL}}$, where $\sigma_{\text{KL}}$ is the empirical standard deviation of the KL divergence between the learned policy and the behavior policy at all the states in the dataset.

\paragraph{Policy network}
Our policy network is a 3-layer feed-forward neural network. The size of each hidden layer is 512. We apply RELU activation \citep{relu} after each hidden layer. Following \citep{sac}, the output is a Gaussian distribution with diagonal covariance matrix. We apply \texttt{tanh} to enforce the action bounds. The log-likelihood after applying the \texttt{tanh} function has a simple closed form solution. We refer to \citep{sac} Appendix C for more details.

\paragraph{Q network}
Following \citep{sac, td3,wu2019behavior}, we train two independent Q network $\{Q_{\psi_1}, Q_{\psi_2}\}$ to penalize uncertainty over the future states. We maintain a target Q network $\{Q_{{\psi'}_1}, Q_{{\psi'}_2}\}$ with the same architecture and update the target weights using a weighted sum of the current Q network and the target Q network. When computing the target Q values, we simply take the minimum value of the two Q networks:
\begin{align}
    Q_{\psi'}(s',a')=\min_{j=1,2}Q_{{\psi'}_j}(s',a')
\end{align}
Each Q network is a 3-layer feed-forward neural network. The size of each hidden layer is 256. We apply RELU activation \citep{relu} after each hidden layer.

\paragraph{Behavior policy network}
Following the previous work \citep{bcq,bear}, we learn a conditional variational auto-encoder \citep{vae} as our behavior policy network. The encoder takes a pair of states and actions, and outputs a Gaussian latent variable $Z$. The decoder takes sampled latent code $z$ and states, and outputs a mixture of Gaussian distributions. Both the architecture of the encoder and the decoder is a 3-layer feed-forward neural network. The size of each hidden layer is 512. The activation is relu \citep{relu}. To avoid epistemic uncertainty, we train $B$ ensembles of behavior policy networks. At test time, we randomly select one model to perform the calculations. We found $B=3$ is sufficient for all the experiments. We pre-train the behavior policy network for 400k gradient steps.

\paragraph{Hyperparameter selection}
The hyperparameters used in our method include the target entropy of the policy $\mathcal{H}_0$ and the divergence threshold $\epsilon$. If the target policy entropy is larger than the behavior policy entropy, the OOD actions are naturally included. In our implementation, we set the $\mathcal{H}_0=\nicefrac{\mathcal{H}(\pi_b)}{4}$ for all the Gym tasks \cite{d4rl}. To set $\epsilon$, we first find $\epsilon_{\min}=\min D(\pi_{\theta}, \pi_b)$. Then, $\epsilon=\epsilon_{\min}+\epsilon_{\text{generalization}}$, where $\epsilon_{\text{generalization}}$ indicates the OOD generalization for this dataset. For all the mixed and random datasets, we set $\epsilon_{\text{generalization}}=3.0$. For the medium-expert datasets and hopper-medium, we set $\epsilon_{\text{generalization}}=0.2$. We set $\epsilon_{\text{generalization}}=1.0$ for walker2d-medium dataset and $\epsilon_{\text{generalization}}=7.0$ for halfcheetah-medium dataset. The default hyperparameters are shown in Table~\ref{table:hyper-parameters}.

% \paragraph{$\alpha$ network}
% The $\alpha$ network takes in a state and outputs the Lagrange multiplier for the state. The architecture of the $\alpha$ network is a 3-layer feed-forward neural network with relu \citep{relu} activation. The size of each hidden layer is 256. We use \texttt{softplus} activation after the output to ensure that all the values are positive.

% \paragraph{Selecting $\epsilon_{\text{KL}}$}
% By nature, hyper-parameter tuning for offline reinforcement learning is hard.

\begin{table}[!ht]
    \centering
    \caption{Default hyper-parameters}
    % \vspace{-1em}
    \begin{tabular}{cc}
        \toprule
        Hyper-parameter & Value \\
        \midrule
        Optimizer & Adam \citep{adam} \\
        Policy learning rate & \texttt{5e-6} \\
        Q network learning rate & \texttt{3e-4} \\
        % $\alpha$ learning rate & \texttt{1e-5/1e-7} \\
        batch size & 100 \\
        Target update rate $\tau$ & \texttt{1e-3} \\
        Discount factor $\gamma$ & 0.99 \\
        % Initial $\beta$  & 10 \\
        % $\beta$ learning rate & \texttt{1e-3} \\
        Steps per epoch $T$ & 2000 \\
        Number of epochs & 500 \\
        \bottomrule
    \end{tabular}
    \label{table:hyper-parameters}
\end{table}

% \begin{table}[!ht]
%     \centering
%     \caption{Task-specific hyper-parameters for OpenAI gym tasks.}
%     \vspace{-1em}
%     \begin{tabular}{ccc}
%         \toprule
%         Task name & KL divergence threshold $\epsilon_{\text{KL}}$ & Maximum entropy $\mathcal{H}_0$ \\
%         \midrule
%         halfcheetah-rand & 9 & -3 \\
%         walker2d-rand  &  0.1 & 6 \\
%         hopper-rand & 3 & -3 \\
%         \midrule
%         halfcheetah-med & 4 & -12 \\
%         walker2d-med  & 2.1 & -9 \\
%         hopper-med & 2.4 &  -6 \\
%         \midrule 
%         % halfcheetah-exp \\
%         % walker2d-exp \\
%         % hopper-exp \\
%         % \midrule
%         halfcheetah-med-exp & 11.5 & -24 \\
%         walker2d-med-exp & 5 & -12 \\
%         hopper-med-exp & 2.6 & -6 \\
%         \midrule
%         halfcheetah-mixed & 6 & -12 \\
%         walker2d-mixed & 4 & -6 \\
%         hopper-mixed & 3 & -3 \\
%         \bottomrule
%     \end{tabular}
% \end{table}

% \begin{table}[!ht]
%     \centering
%     \caption{Task-specific hyper-parameters for Adroit tasks with human demonstrations.}
%     \vspace{-1em}
%     \begin{tabular}{ccc}
%         \toprule
%         Task name & MMD threshold $\epsilon_{\text{MMD}}$ & Minimum entropy $\mathcal{H}_0$ \\
%         \midrule
%         pen-human & 0.06 & -200 \\
%         hammer-human & 0.1 & -60 \\
%         door-human & 0.1 & -60 \\
%         relocate-human & 0.1 & -60 \\
%         \bottomrule
%     \end{tabular}
% \end{table}

\end{document}